\documentclass{article}

\usepackage[nonatbib, final]{./NeuRIPS2019/neurips_2019}
\usepackage{natbib}
\usepackage[utf8]{inputenc} 
\usepackage[T1]{fontenc}    
\usepackage{hyperref}       
\usepackage{url}            
\usepackage{booktabs}       
\usepackage{amsthm}
\usepackage{amsmath}
\usepackage{amsfonts}       
\usepackage{nicefrac}       
\usepackage{microtype}      
\usepackage{xspace}

\usepackage{algorithm}
\usepackage{algpseudocode}
\usepackage{mathtools}
\usepackage{appendix}

\usepackage{graphicx}
\usepackage{subfig}
\usepackage{color}
\usepackage{bm}

\usepackage{tikz}
\usetikzlibrary{arrows,patterns}
\usetikzlibrary{positioning}
\usetikzlibrary{backgrounds}
\usetikzlibrary{arrows.meta}

\algnewcommand{\IIf}[1]{\State\algorithmicif\ #1\ \algorithmicthen}
\algnewcommand{\EndIIf}{\unskip\ \algorithmicend\ \algorithmicif}
\newtheorem{prop}{Proposition}

\newcommand{\GALA}{\textsc{gala}\xspace}
\newcommand{\IMPALA}{\textsc{impala}\xspace}
\newcommand{\AtwC}{\textsc{a2c}\xspace}
\newcommand{\AthC}{\textsc{a3c}\xspace}

\newcommand{\Hogwild}{\textsc{Hogwild!}\xspace}
\newcommand{\AllReduce}{\textsc{AllReduce}\xspace}

\newcommand{\R}{\mathbb{R}}
\newcommand{\N}{\mathbb{N}}
\newcommand{\ie}{i.e.}
\newcommand{\eg}{e.g.}
\newcommand{\defeq}{\coloneqq}

\newcommand{\ones}{\bm{1}}

\DeclareMathOperator*{\argmax}{arg\,max}

\newcommand{\norm}[1]{\left\Vert #1 \right\Vert}
\newcommand{\sspan}[1]{\text{span}\left\{ #1 \right\}}

\newcommand{\param}{x}

\newcommand{\Nout}{\mathcal{N}^{\text{out}}}
\newcommand{\Nin}{\mathcal{N}^{\text{in}}}
\newcommand{\Bbuff}{\mathcal{B}}
\newcommand{\Rbuff}{\mathcal{R}}
\newcommand{\Tfunc}{\mathcal{T}}
\newcommand{\naug}{\tilde{n}}

\usepackage{enumitem}

\title{Gossip-based Actor-Learner Architectures for Deep Reinforcement Learning}

\author{
    Mahmoud Assran \\
    Facebook AI Research \& \\
    Department of Electrical and Computer Engineering \\
    McGill University \\
    \texttt{mahmoud.assran@mail.mcgill.ca}
    \And
    Joshua Romoff \\
    Facebook AI Research \& \\
    Department of Computer Science \\
    McGill University \\
    \texttt{joshua.romoff@mail.mcgill.ca} \\    
    \And
    Nicolas Ballas \\
    Facebook AI Research \\
    \texttt{ballasn@fb.com} \\    
    \And
    Joelle Pineau \\
    Facebook AI Research \\
    \texttt{jpineau@fb.com} \\    
    \And
    Michael Rabbat \\
    Facebook AI Research \\
    \texttt{mikerabbat@fb.com}
}

\begin{document}
    
\maketitle

\begin{abstract}
Multi-simulator training has contributed to the recent success of Deep Reinforcement Learning by stabilizing learning and allowing for higher training throughputs.
We propose Gossip-based Actor-Learner Architectures (\GALA) where several actor-learners (such as \AtwC agents) are organized in a peer-to-peer communication topology, and exchange information through asynchronous gossip in order to take advantage of a large number of distributed simulators.
We prove that \GALA agents remain within an $\epsilon$-ball of one-another during training when using loosely coupled asynchronous communication.
By reducing the amount of synchronization between agents, \GALA is more computationally efficient and scalable compared to \AtwC, its fully-synchronous counterpart.
\GALA also outperforms \AthC, being more robust and sample efficient.
We show that we can run several loosely coupled \GALA agents in parallel on a single GPU and achieve significantly higher hardware utilization and frame-rates than vanilla \AtwC at comparable power draws.

\end{abstract}

\section{Introduction}

Deep Reinforcement Learning (Deep RL) agents have reached superhuman performance in a few domains~\citep{silver2016mastering,silver2018general,mnih2015human,vinyals2019alphastar}, but this is typically at significant computational expense~\citep{tian2019elf}.
To both reduce running time and stabilize training, current approaches rely on distributed computation wherein data is sampled from many parallel simulators distributed over parallel devices~\citep{espeholt2018impala, mnih2016asynchronous}.
Despite the growing ubiquity of multi-simulator training, scaling Deep RL algorithms to a large number of simulators remains a challenging task.

On-policy approaches train a policy by using samples generated from that same policy, in which case data sampling (acting) is entangled with the training procedure (learning).
To perform distributed training, these approaches usually introduce multiple learners with a \emph{shared policy}, and multiple actors (each with its own simulator) associated to each learner.
The shared policy can either be updated in a synchronous fashion (\eg, learners synchronize gradients before each optimization step~\citep{stooke2018accelerated}), or in an asynchronous fashion~\citep{mnih2016asynchronous}.
Both approaches have drawbacks: synchronous approaches suffer from straggler effects (bottlenecked by the slowest individual simulator), and therefore may not exhibit strong scaling efficiency; asynchronous methods are robust to stragglers, but prone to gradient staleness, and may become unstable with a large number of actors~\citep{clemente2017efficient}. 

Alternatively, off-policy approaches typically train a policy by sampling from a replay buffer of past transitions~\citep{mnih2015human}.
Training off-policy allows for disentangling data-generation from learning, which can greatly increase computational efficiency when training with many parallel actors~\citep{espeholt2018impala,apex2018,kapturowski2018recurrent, reactor}.
Generally, off-policy updates need to be handled with care as the sampled transitions may not conform to the current policy and consequently result in unstable training~\citep{fujimoto2018addressing}. 

We propose Gossip-based Actor-Learner Architectures (\GALA), which aim to retain the robustness of synchronous on-policy approaches, while improving both their computational efficiency and scalability. \GALA leverages multiple agents, where each agent is composed of one learner and possibly multiple actors/simulators. Unlike classical on-policy approaches, \GALA does not require that each agent share the same policy, but rather it inherently enforces (through gossip) that each \emph{agent's policy remain $\epsilon$-close to all others throughout training}. Relaxing this constraint allows us to reduce the synchronization needed between learners, thereby improving the algorithm's computational efficiency. 

Instead of computing an exact average between all the learners after a local optimization step, gossip-based approaches compute an approximate average using loosely coupled and possibly asynchronous communication (see~\cite{nedic2018network} and references therein). While this approximation implicitly injects some noise in the aggregate parameters, we prove that this is in fact a principled approach as the learners' policies stay within an $\epsilon$-ball of one-another (even with non-linear function approximation), the size of which is directly proportional to the spectral-radius of the agent communication topology and their learning rates.

As a practical algorithm, we propose \GALA-\AtwC, an algorithm that combines gossip with \AtwC agents. We compare our approach on six Atari games~\citep{machado17arcade} following~\cite{stooke2018accelerated} with vanilla \AtwC, \AthC and the \IMPALA  off-policy method~\citep{baselines,mnih2016asynchronous,espeholt2018impala}. Our main empirical findings are:
\begin{enumerate}
    \item Following the theory, \GALA-\AtwC is empirically stable. Moreover, we observe  that \GALA can be more stable than \AtwC when using a large number of simulators, suggesting that the noise introduced by gossiping can have a beneficial effect.
    \item \GALA-\AtwC has similar sample efficiency to \AtwC and greatly improves its computational efficiency and scalability.
    \item \GALA-\AtwC achieves significantly higher hardware utilization and frame-rates than vanilla \AtwC at comparable power draws, when using a GPU.
    \item \GALA-\AtwC is competitive in term of performance relative to \AthC and \IMPALA.
\end{enumerate}
Perhaps most remarkably, our empirical findings for \GALA-\AtwC are obtained by simply using the default hyper-parameters from \AtwC.
Our implementation of \GALA-\AtwC is publicly available at ~\href{https://github.com/facebookresearch/gala}{https://github.com/facebookresearch/gala}.
\section{Technical Background}
\paragraph{Reinforcement Learning.}
We consider the standard Reinforcement Learning setting \citep{sutton1998introduction}, where the agent’s objective is to maximize the expected value from each state $V(s) = \mathbb{E} \left [ \sum_{i=0}^{\infty} \gamma^i r_{t+i} | s_t=s \right ]$, $\gamma$ is the discount factor which controls the bias towards nearby rewards. To maximize this quantity, the agent chooses at each discrete time step $t$ an action $a_t$  in the current state $s_t$ based on its policy $\pi(a_t|s_t)$ and transitions to the next state $s_{t+1}$ receiving reward $r_t$ based on  the environment dynamics.

Temporal difference (TD) learning \citep{sutton1984temporal} aims at learning an approximation of the expected return parameterized by $\theta$, \ie, the value function $V(s;\theta)$, by iteratively updating its parameters via gradient descent:
\begin{equation}
\label{eq:td}
	\nabla_{\theta} \left(G^N_t- V(s_t;\theta) \right) ^2
\end{equation} 
 where $G^N_t=\sum_{i=0}^{N-1} \gamma^i r_{t+i} + \gamma^N  V(s_{t+n};\theta_t)$ is the $N$-step return.
Actor-critic methods \citep{sutton2000policy, mnih2016asynchronous} simultaneously learn both a parameterized policy $\pi(a_t|s_t; \omega)$ with parameters $\omega$ and a critic $V(s_t;\theta)$. They do so by training a value function via the TD error defined in~\eqref{eq:td} and then proceed to optimize the policy using the policy gradient (PG) with the value function as a baseline:
\begin{equation}
\label{eq:acgrad}
	\nabla_\omega \left(- \log \pi(a_t|s_t;\omega) A_t\right),
\end{equation}
where $A_t = G^N_t - V(s_t; \theta_t)$ is the advantage function, which represents the relative value the current action has over the average. 
In order to both speed up training time and decorrelate observations,~\cite{mnih2016asynchronous} collect samples and perform updates with several asynchronous actor-learners.
Specifically, each worker $i\in\{1,2, .., W\}$, where $W$ is the number of parallel workers, collects samples according to its current version of the policy weights $\omega_i$, and computes updates via the standard actor-critic gradient defined in~\eqref{eq:acgrad}, with an additional entropy penalty term that prevents premature convergence to deterministic policies:
\begin{equation}
    \nabla_{\omega_i} \left(- \log \pi(a_t|s_t;\omega_i) A_t - \eta \sum_a  \pi(a|s_t; \omega_i) \log \pi(a|s_t; \omega_i)\right).
\end{equation}
The workers then perform \Hogwild~\citep{recht2011hogwild} style updates (asynchronous writes) to a shared set of master weights before synchronizing their weights with the master's.
More recently, \cite{baselines} removed the asynchrony from \AthC, referred to as \AtwC, by instead synchronously collecting transitions in parallel environments $i\in\{1,2, .., W\}$ and then performing a large batched update:
\begin{equation}
   \nabla_\omega \left [\frac{1}{W}\sum_{i=1}^W \left (- \log \pi(a^i_t|s^i_t;\omega) A^i_t - \eta \sum_a  \pi(a|s^i_t; \omega) \log \pi(a|s^i_t; \omega) \right ) \right].
\label{eq:a2c}
\end{equation}

\paragraph{Gossip algorithms.}

Gossip algorithms are used to solve the distributed averaging problem.
Suppose there are $n$ agents connected in a peer-to-peer graph topology, each with parameter vector $\param^{(0)}_i \in \R^d$.
Let $\bm{X}^{(0)} \in \R^{n \times d}$ denote the row-wise concatenation of these vectors.
The objective is to iteratively compute the average vector $\frac{1}{n} \sum_{i=1}^n \bm{x}_i^{(0)}$ across all agents.
Typical gossip iterations have the form $\bm{X}^{(k+1)} = \bm{P}^{(k)} \bm{X}^{(k)}$, where $\bm{P}^{(k)} \in \R^{n \times n}$ is referred to as the mixing matrix and defines the communication topology.
This corresponds to the update $\bm{x}_i^{(k+1)} = \sum_{j=1}^n p_{i,j}^{(k)} \bm{x}_j^{(k)}$ for an agent $v_i$.
At an iteration $k$,  an agent $v_i$ only needs to receive messages from other agents $v_j$ for which $p_{i,j}^{(k)} \ne 0$, so sparser matrices $\bm{P}^{(k)}$ correspond to less communication and less synchronization between agents.

The mixing matrices $\bm{P}^{(k)}$ are designed to be row stochastic (each entry is greater than or equal to zero, and each row sums to 1) so that $\lim_{K \rightarrow \infty} \prod_{k=0}^{K} \bm{P}^{(k)} = \ones \bm{\pi}^\top$, where $\bm{\pi}$  is the ergodic limit of the Markov chain defined by $\bm{P}^{(k)}$  and $\ones$ is a vector with all entries equal to $1$~\citep{Seneta1981}.\footnote{Assuming that information from every agent eventually reaches all other agents }
Consequently, the gossip iterations converge to a limit $\bm{X}^{(\infty)} = \ones (\bm{\pi} ^\top \bm{X}^{(0)})$; meaning the value at an agent $i$ converges to $\bm{x}_i^{(\infty)} = \sum_{j=1}^n \pi_j \bm{x}_j^{(0)}$. 
In particular, if the matrices $\bm{P}^{(k)}$ are symmetric and doubly-stochastic (each row and each column must sum to 1), we obtain an algorithm such that $\pi_j = 1/n$ for all $j$, and therefore $\bm{x}_i^{(\infty)} = 1/n \sum_{j=1}^n \bm{x}_j^{(0)}$ converges to the average of the agents' initial vectors. 

For the particular case of \GALA, we only require the matrices $\bm{P}^{(k)}$ to be row stochastic in order to show the $\epsilon$-ball guarantees.

\section{Gossip-based Actor-Learner Architectures}

\begin{algorithm}[t]
    \small
	\caption{Gossip-based Actor-Learner Architectures for agent $v_i$ using \AtwC} \label{alg:gala}
  	\begin{algorithmic}[1]
  	    \Require{Initialize trainable policy and critic parameters $\param_i = (\omega_i, \theta_i)$.}
    	\For{$t$ = 0, 1, 2, \dots}
    	    \State {Take $N$ actions $\{ a_t \}$ according to $\pi_{\omega_i}$ and store transitions $\{(s_t, a_t. r_t, s_{t+1})\}$}
    	    \State {Compute returns $G_t^N=\sum_{i=0}^{N-1} \gamma^i r_{t+i} + \gamma^N  V(s_{t+n};\theta_i)$ and advantages $A_t= G^N_t - V(s_t; \theta_i)$}
	        \State {Perform \AtwC optimization step on $\param_i$ using TD in~\eqref{eq:td} and batched policy-gradient in~\eqref{eq:a2c}}
	        \State {Broadcast (non-blocking) new parameters $\param_i$ to all out-peers in $\Nout_i$}
	        \If {Receive buffer contains a message $m_j$ from each in-peer $v_j$ in $\Nin_i$}
	            \State {$\param_i \gets \frac{1}{1 + |\Nin_i|}( \param_i + \sum_{j} m_j)$\footnotemark[1]} \Comment{Average parameters with messages}
                \label{alg:gala:line:equal-neighbor-iteration}
            \EndIf
 		\EndFor
	\end{algorithmic}
	\begin{flushleft}
    $^{1}$ \footnotesize{ We set the non-zero mixing weights for agent $v_i$ to $p_{i,j} = \frac{1}{1 + |\Nin_i|}$.}\\
\end{flushleft}
\end{algorithm}
We consider the distributed RL setting where $n$ agents (each composed of a single learner and several actors) collaborate to maximize the expected return $V(s)$.
Each agent $v_i$
has a parameterized policy network $\pi(a_t|s_t; \omega_i)$ and value function $V(s_t;\theta_i)$.
Let $\param_i = (\omega_i, \theta_i)$ denote agent $v_i$'s complete set of trainable parameters.
We consider the specific case where each $v_i$ corresponds to a single \AtwC agent, and the agents are configured in a directed and peer-to-peer communication topology defined by the mixing matrix $\bm{P} \in \R^{n \times n}$.



In order to maximize the expected reward, each \GALA-\AtwC agent alternates between one local policy-gradient and TD update, and one iteration of asynchronous gossip with its peers.
Pseudocode is provided in Algorithm~\ref{alg:gala}, where $\Nin_i \defeq \{ v_j \mid p_{i,j} > 0 \}$ denotes the set of agents that send messages to agent $v_i$ (in-peers), and $\Nout_i \defeq \{ v_j \mid p_{j,i} > 0 \}$ the set of agents that $v_i$ sends messages to (out-peers). 
During the gossip phase, agents broadcast their parameters to their out-peers, asynchronously (\ie, don't wait for messages to reach their destination), and update their own parameters via a convex combination of all received messages.
Agents broadcast new messages when old transmissions are completed and aggregate all received messages once they have received a message from each in-peer.

Note that the \GALA agents use non-blocking communication, and therefore operate asynchronously.
Local iteration counters may be out-of-sync, and physical message delays may result in agents incorporating outdated messages from their peers.
One can algorithmically enforce an upper bound on the message staleness by having the agent block and wait for communication to complete if more than $\tau \geq 0$ local iterations have passed since the agent last received a message from its in-peers.



\paragraph{Theoretical $\epsilon$-ball guarantees:} Next we provide the $\epsilon$-ball theoretical guarantees for the asynchronous $\GALA$ agents, proofs of which can be found in Appendix~\ref{apdx:proofs}.
Let $k \in \N$ denote the global iteration counter, which increments whenever any agent (or subset of agents) completes an iteration of the loop defined in Algorithm~\ref{alg:gala}.
We define $\param_i^{(k)} \in \R^d$ as the value of agent $v_i$'s trainable parameters at iteration $k$, and $\bm{X}^{(k)} \in \R^{n \times d}$ as the row-concatenation of these parameters.

For our theoretical guarantees we let the communication topologies be directed and time-varying graphs, and we do not make any assumptions about the base \GALA learners.
In particular, let the mapping $\Tfunc_i : \param_i^{(k)} \in \R^d \mapsto \param_i^{(k)} - \alpha g_i^{(k)} \in \R^d$ characterize agent $v_i$'s local training dynamics (\ie, agent $v_i$ optimizes its parameters by computing $\param_i^{(k)} \gets \Tfunc_i(\param_i^{(k)})$), where $\alpha > 0$ is a reference learning rate, and $g_i^{(k)} \in \R^d$ can be any update vector.
Lastly, let $\bm{G}^{(k)} \in \R^{n \times d}$ denote the row-concatenation of these update vectors.

\label{sec:theory}
\begin{prop}
\label{prop:no-ass}
For all $k \geq 0$, it holds that
\[
    \norm{ \bm{X}^{(k + 1)} - \overline{\bm{X}}^{(k+1)} } \leq \alpha \sum^k_{s=0} \beta^{k + 1 - s} \norm{\bm{G}^{(s)}},
\]
where $\overline{\bm{X}}^{(k+1)} \defeq \frac{1_{n} 1_{n}^T}{n} \bm{X}^{(k + 1)}$ denotes the average of the learners' parameters at iteration $k + 1$, and $\beta \in [0, 1]$ is related to the joint spectral radius of the graph sequence defining the communication topology at each iteration.
\end{prop}

Proposition~\ref{prop:no-ass} shows that the distance of a learners' parameters from consensus is bounded at each iteration.
However, without additional assumptions on the communication topology, the constant $\beta$ may equal $1$, and the bound in Proposition~\ref{prop:no-ass} can be trivial.
In the following proposition, we make sufficient assumptions with respect to the graph sequence that ensure $\beta < 1$.

\begin{prop}
\label{prop:ass}
Suppose there exists a finite integer $B \geq 0$ such that the (potentially time-varying) graph sequence is $B$-strongly connected, and suppose that the upper bound $\tau$ on the message delays in Algorithm~\ref{alg:gala} is finite.
If learners run Algorithm~\ref{alg:gala} from iteration $0$ to $k + 1$, where $k \geq \tau + B$, then it holds that
\[
    \norm{ \bm{X}^{(k + 1)} - \overline{\bm{X}}^{(k+1)}} \leq \frac{\alpha \tilde{\beta} L}{1 - \beta},
\]
where $\beta < 1$ is related to the joint spectral radius of the graph sequence, $\alpha$ is the reference learning rate, $\tilde{\beta} \defeq \beta^{- \frac{\tau + B}{\tau + B + 1}}$, and $L \defeq \sup_{s=1,2,\ldots} \norm{\bm{G}^{(s)}}$ denotes an upper bound on the magnitude of the local optimization updates during training.
\end{prop}

Proposition~\ref{prop:ass} states that the agents' parameters are guaranteed to reside within an $\epsilon$-ball of their average at all iterations $k \geq \tau + B$.
The size of this ball is proportional to the reference learning-rate, the spectral radius of the graph topology, and the upper bound on the magnitude of the local gradient updates.
One may also be able to control the constant $L$ in practice since Deep RL agents are typically trained with some form of gradient clipping.

\section{Related work}

Several recent works have approached scaling up RL by using parallel environments. \cite{mnih2016asynchronous} used parallel asynchronous agents to perform \Hogwild~\citep{recht2011hogwild} style updates to a shared set of parameters. \cite{baselines} proposed \AtwC, which maintains the parallel data collection, but performs updates synchronously, and found this to be more stable empirically.
While A3C was originally designed as a purely CPU-based method,~\cite{babaeizadeh2016ga3c} proposed GA3C, a GPU implementation of the algorithm.
\cite{stooke2018accelerated} also scaled up various RL algorithms by using significantly larger batch sizes and distributing computation onto several GPUs. Differently from those works, we propose the use of \emph{Gossip Algorithms} to aggregate information between different agents and thus simulators.
\cite{gorillaxb, apex2018, espeholt2018impala, kapturowski2018recurrent, reactor} use parallel environments as well, but disentangle the data collection (actors) from the network updates (learners). This provides several computational benefits, including better hardware utilization and reduced straggler effects. By disentangling acting from learning these algorithms must use off-policy methods to handle learning from data that is not directly generated from the current policy (e.g., slightly older policies).


Gossip-based approaches have been extensively studied in the control-systems literature as a way to aggregate information for distributed optimization algorithms ~\citep{nedic2018network}. In particular, recent works have proposed to combine gossip algorithms with stochastic gradient descent in order to train Deep Neural Networks~\citep{lian2017asynchronous,lian2017can,assran2018stochastic}, but unlike our work, focus only on the supervised classification paradigm.
 

\section{Experiments}
We evaluate \GALA for training Deep RL agents on Atari-2600 games~\citep{machado17arcade}. We focus on the same six games studied in \cite{stooke2018accelerated}.
Unless otherwise-stated, all learning curves show averages over 10 random seeds with $95\%$ confidence intervals shaded in. We follow the reproducibility checklist \citep{repro_checklist}, see Appendix~\ref{app:repro} for details.

We compare \AtwC~\citep{baselines}, \AthC~\citep{mnih2016asynchronous}, \IMPALA~\citep{espeholt2018impala}, and \GALA-\AtwC.
All methods are implemented in PyTorch~\citep{paszke2017automatic}.
While \AthC was originally proposed with CPU-based agents with 1-simulator per agent,~\citet{stooke2018accelerated} propose a large-batch variant in which each agent manages 16-simulators and performs batched inference on a GPU.
We found this large-batch variant to be more stable and computationally efficient (cf.~Appendix~\ref{app:a3c}).
We use the~\cite{stooke2018accelerated} variant of \AthC to provide a more competitive baseline.
We parallelize \AtwC training via the canonical approach outlined in~\cite{stooke2018accelerated}, whereby individual \AtwC agents (running on potentially different devices), all average their gradients together before each update using the \AllReduce primitive.\footnote{This is mathematically equivalent to a single \AtwC agent with multiple simulators (\eg, $n$ agents, with $b$ simulators each, are equivalent to a single agent with $nb$ simulators).}
For \AtwC and \AthC we use the hyper-parameters suggested in~\cite{stooke2018accelerated}.
For \IMPALA we use the hyper-parameters suggested in~\cite{espeholt2018impala}.
For \GALA-\AtwC we use the same hyper-parameters as the original (non-gossip-based) method.
All \GALA agents are configured in a directed ring graph.
All implementation details are described in Appendix~\ref{app:implementation}.
For the \IMPALA baseline, we use a prerelease of TorchBeast~\citep{kuttler2019torchbeast} available at~\href{https://github.com/facebookresearch/torchbeast}{https://github.com/facebookresearch/torchbeast}.

\paragraph{Convergence and stability:}

We begin by empirically studying the convergence and stability properties of \AtwC and \GALA-\AtwC.
Figure~\ref{fig:sim_sweep} depicts the percentage of successful runs (out of 10 trials) of standard policy-gradient \AtwC when we sweep the number of simulators across six different games.
We define a run as successful if it achieves better than $50\%$ of nominal $16$-simulator \AtwC scores.
When using \AtwC, we observe an identical trend across all games in which the number of successful runs decreases significantly as we increase the number of simulators.
Note that the \AtwC batch size is proportional to the number of simulators, and when increasing the number of simulators we adjust the learning rate following the recommendation in~\cite{stooke2018accelerated}.

\begin{figure}
    \centering
    \subfloat[Simulator sweep: Percentage of convergent runs out of $10$ trials.
                \label{fig:sim_sweep}]{\includegraphics[width=0.95\textwidth]{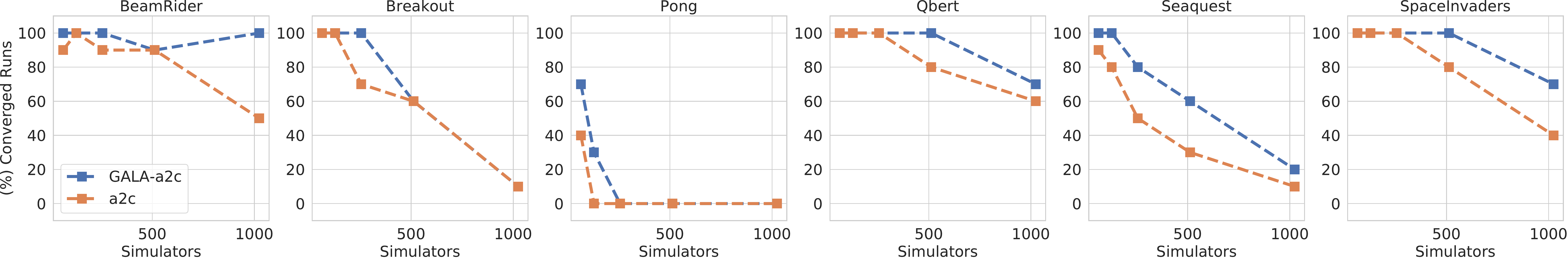}} \\
    \subfloat[Sample complexity: Best 3 runs for each method.
                \label{fig:best_64sim_itr}]{\includegraphics[width=0.95\textwidth]{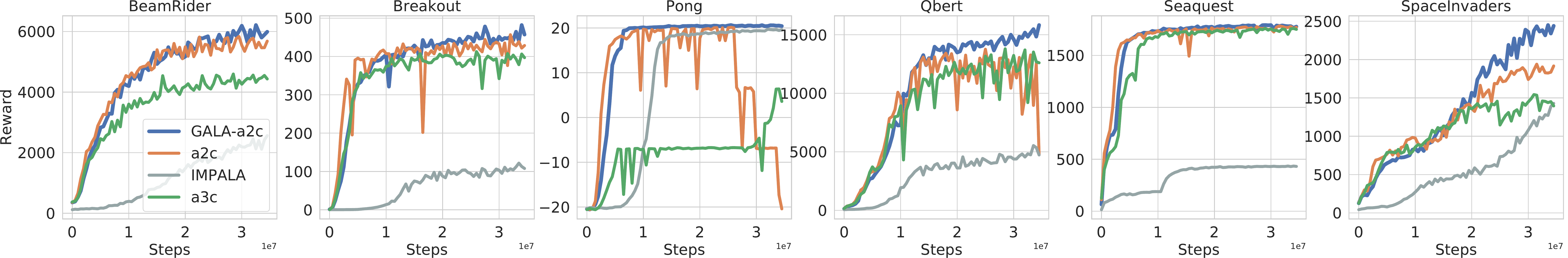}} \\
    \subfloat[Sample complexity: Average across 10 runs.
                \label{fig:64sim_itr}]{\includegraphics[width=0.95\textwidth]{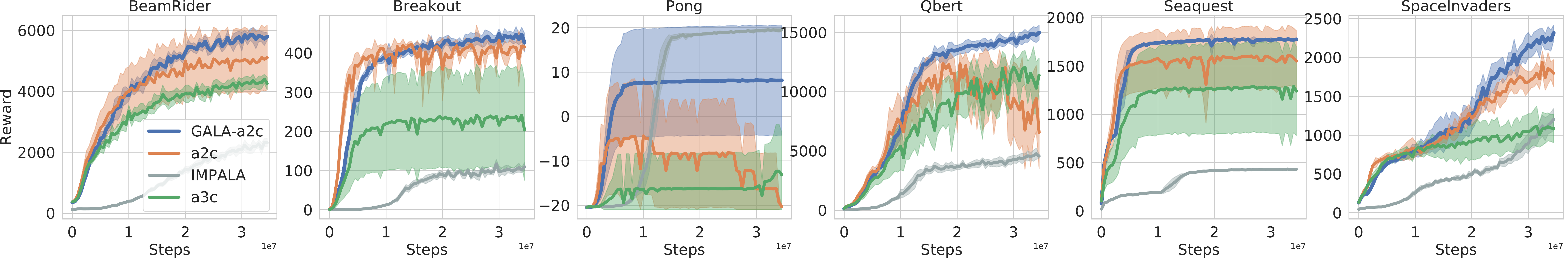}} \\
    \subfloat[Computational complexity: Best 3 runs for each method.
                \label{fig:best_64sim_time}]{\includegraphics[width=0.95\textwidth]{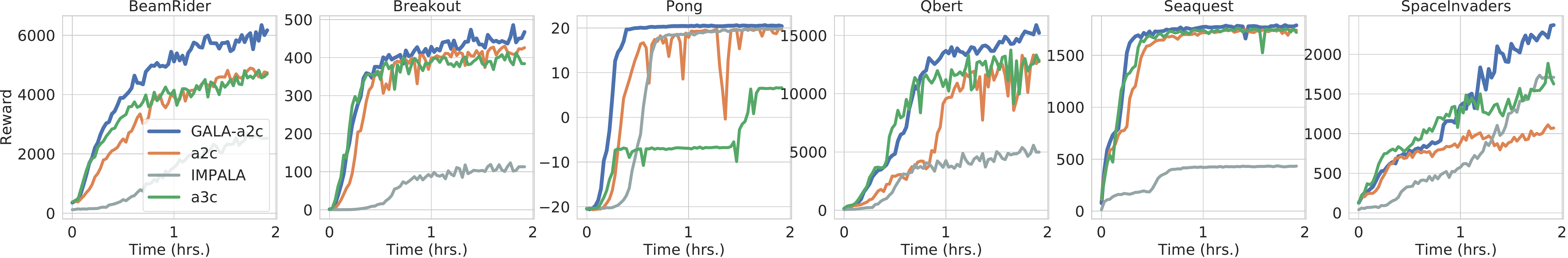}} \\
    \subfloat[Computational complexity: Average across 10 runs. 
                \label{fig:64sim_time}]{\includegraphics[width=0.95\textwidth]{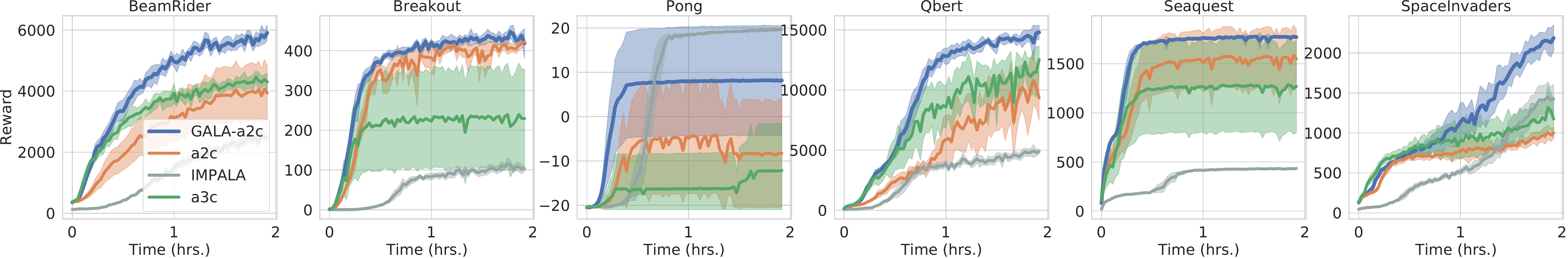}} \\
    \subfloat[Energy efficiency: Best 3 runs for each method. \label{fig:best_64sim_energy}]{\includegraphics[width=0.95\textwidth]{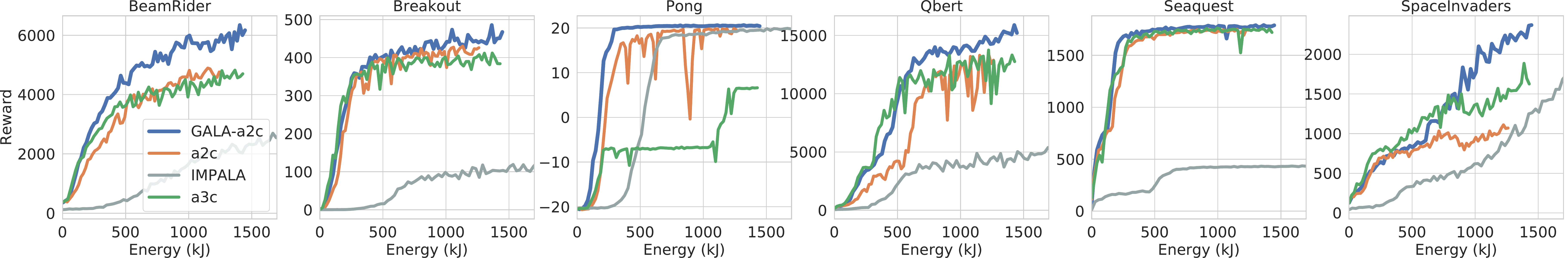}} \\
    \subfloat[Energy efficiency: Average across 10 runs. \label{fig:64sim_energy}]{\includegraphics[width=0.95\textwidth]{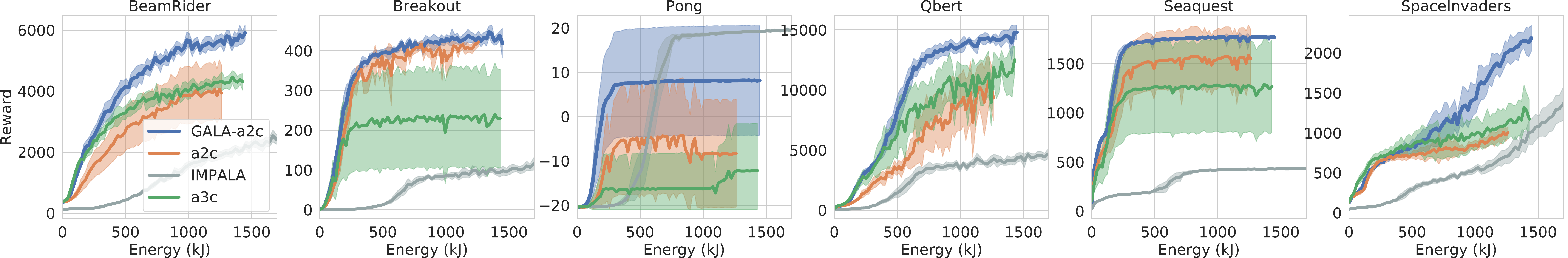}}
    \caption{
    (a) \GALA increases or maintains the percentage of convergent runs relative to \AtwC.
    (b)-(c) \GALA maintains the best performance of \AtwC while being more robust.
    (d)-(e) \GALA achieves competitive scores in each game and in the shortest amount of time.
    (f)-(g) \GALA achieves competitive game scores while being energy efficient.
    }
    \label{fig:64sim_all}
\end{figure}

Figure~\ref{fig:sim_sweep} also depicts the percentage of successful runs when \AtwC agents communicate their parameters using gossip algorithms (\GALA-\AtwC).
In \emph{every} simulator sweep across the six games (600 runs), the gossip-based architecture increases or maintains the percentage of successful runs relative to vanilla \AtwC, when using identical hyper-parameters.
We hypothesize that exercising slightly different policies at each learner using gossip-algorithms can provide enough decorrelation in gradients to improve learning stability.
We revisit this point later on (cf.~Figure~\ref{subfig:hw-bar}).
We note that~\cite{stooke2018accelerated} find that stepping through a random number of uniform random actions at the start of training can partially mitigate this stability issue.
We did not use this random start action mitigation in the reported experiments.

While Figure~\ref{fig:sim_sweep} shows that \GALA can be used to stabilize multi-simulator \AtwC and increase the number of successfull runs, it does not directly say anything about the final performance of the learned models.
Figures~\ref{fig:best_64sim_itr} and~\ref{fig:64sim_itr} show the rewards plotted against the number of environment steps when training with 64 simulators.
Using gossip-based architectures stabilizes and maintains the peak performance and sample efficiency of \AtwC across all six games (Figure~\ref{fig:best_64sim_itr}), and also increases the number of convergent runs (Figure~\ref{fig:64sim_itr}).

Figures~\ref{fig:best_64sim_time} and~\ref{fig:64sim_time} compare the wall-clock time convergence of \GALA-\AtwC to vanilla \AtwC.
Not only is \GALA-\AtwC more stable than \AtwC, but it also runs at a higher frame-rate by mitigating straggler effects.
In particular, since \GALA-\AtwC learners do not need to synchronize their gradients, each learner is free to run at its own rate without being hampered by variance in peer stepping times.

\paragraph{Comparison with distributed Deep RL approaches:}

Figure~\ref{fig:64sim_all} also compares \GALA-\AtwC to state-of-the-art methods like \IMPALA and \AthC.\footnote{We report results for both the TorchBeast implementation of \IMPALA, and from Table $C.1$ from \cite{espeholt2018impala}}
In each game, the \GALA-\AtwC learners exhibited good sample efficiency and computational efficiency, and achieved highly competitive final game scores.
\begin{table*}[t]
\caption{
Across all training seeds we select the best final policy produced by each method at the end of training and evaluate it over 10 evaluation episodes (up to $30$ no-ops at the start of the episode).
Evaluation actions generated from $\argmax_a \pi(a|s)$.
The table depicts the mean and standard error across these 10 evaluation episodes.
}
\label{tb:gala_64sim}
\centering
\small
\begin{tabular}{l c l l l l l l} \toprule 
& Steps & \multicolumn{1}{c}{\textbf{BeamRider}} & \multicolumn{1}{c}{\textbf{Breakout}} & \multicolumn{1}{c}{\textbf{Pong}} & \multicolumn{1}{c}{\textbf{Qbert}} & \multicolumn{1}{c}{\textbf{Seaquest}} & \multicolumn{1}{c}{\textbf{SpaceInvaders}} \\ \midrule
 \IMPALA\footnotemark[1] & 50M & 8220 & 641 & \textbf{21} & 18902 & 1717 & 1727 \\ \midrule
 \IMPALA & 40M & 7118 $\scriptstyle \pm 2536$ & 127 $\scriptstyle \pm 65$ & \textbf{21} $\scriptstyle \pm 0$ & 7878 $\scriptstyle \pm 2573$ & 462 $\scriptstyle \pm 2$ & \textbf{4071} $\scriptstyle \pm 393$ \\
 \AthC & 40M & 5674 $\scriptstyle \pm 752$ & 414 $\scriptstyle \pm 56$ & \textbf{21} $\scriptstyle \pm 0$ & 14923 $\scriptstyle \pm 460$ & 1840 $\scriptstyle \pm 0$ & 2232 $\scriptstyle \pm 302$ \\
 \AtwC & 25M & 8755 $\scriptstyle \pm 811$ & 419 $\scriptstyle \pm 3$ & \textbf{21} $\scriptstyle \pm 0$ & 16805 $\scriptstyle \pm 172$ & 1850 $\scriptstyle \pm 5$ & 2846 $\scriptstyle \pm 22$ \\
 \AtwC & 40M & \textbf{9829} $\scriptstyle \pm 1355$ & 495 $\scriptstyle \pm 57$ & \textbf{21} $\scriptstyle \pm 0$ & 19928 $\scriptstyle \pm 99$ & \textbf{1894} $\scriptstyle \pm 6$ & 3021 $\scriptstyle \pm 36$ \\
 \midrule

 \GALA-\AtwC & 25M & \textbf{9500} $\scriptstyle \pm 1020$ & \textbf{690} $\scriptstyle \pm 72$ & \textbf{21} $\scriptstyle \pm 0$ & 18810 $\scriptstyle \pm 37$ & 1874 $\scriptstyle \pm 4$ & 2726 $\scriptstyle \pm 189$ \\
  \GALA-\AtwC & 40M & \textbf{10188} $\scriptstyle \pm 1316$ & \textbf{690} $\scriptstyle \pm 72$ & \textbf{21} $\scriptstyle \pm 0$ & \textbf{20150} $\scriptstyle \pm 28$ & \textbf{1892} $\scriptstyle \pm 6$ & 3074 $\scriptstyle \pm 69$ \\
 \bottomrule
\end{tabular}
\begin{flushleft}
$^{1}$ \footnotesize{\cite{espeholt2018impala} results using shallow network (identical to the network used in our experiments).}
\end{flushleft}
\end{table*}
Next we evaluate the final policies produced by each method at the end of training.
After training across 10 different seeds, we are left with 10 distinct policies per method.
We select the best final policy and evaluate it over 10 evaluation episodes, with actions generated from $\argmax_a \pi(a|s)$.
In almost every single game, the \GALA-\AtwC learners achieved the highest evaluation scores of any method.
Notably, the \GALA-\AtwC learners that were trained for 25M steps achieved (and in most cases surpassed) the scores for \IMPALA learners trained for 50M steps~\citep{espeholt2018impala}.

\paragraph{Effects of gossip:}
To better understand the stabilizing effects of \GALA, we evaluate the diversity in learner policies during training.
Figure~\ref{subfig:epsilon-ball} shows the distance of the agents' parameters from consensus throughout training.
The theoretical upper bound in Proposition~\ref{prop:no-ass} is also explicitly calculated and plotted in Figure~\ref{subfig:epsilon-ball}.
As expected, the learner policies remain within an $\epsilon$-ball of one-another in weight-space, and this size of this ball is remarkably well predicted by Proposition~\ref{prop:no-ass}.

\begin{figure}[t]
    \centering
    \subfloat[\label{subfig:epsilon-ball}]{\includegraphics[width=.25\textwidth]{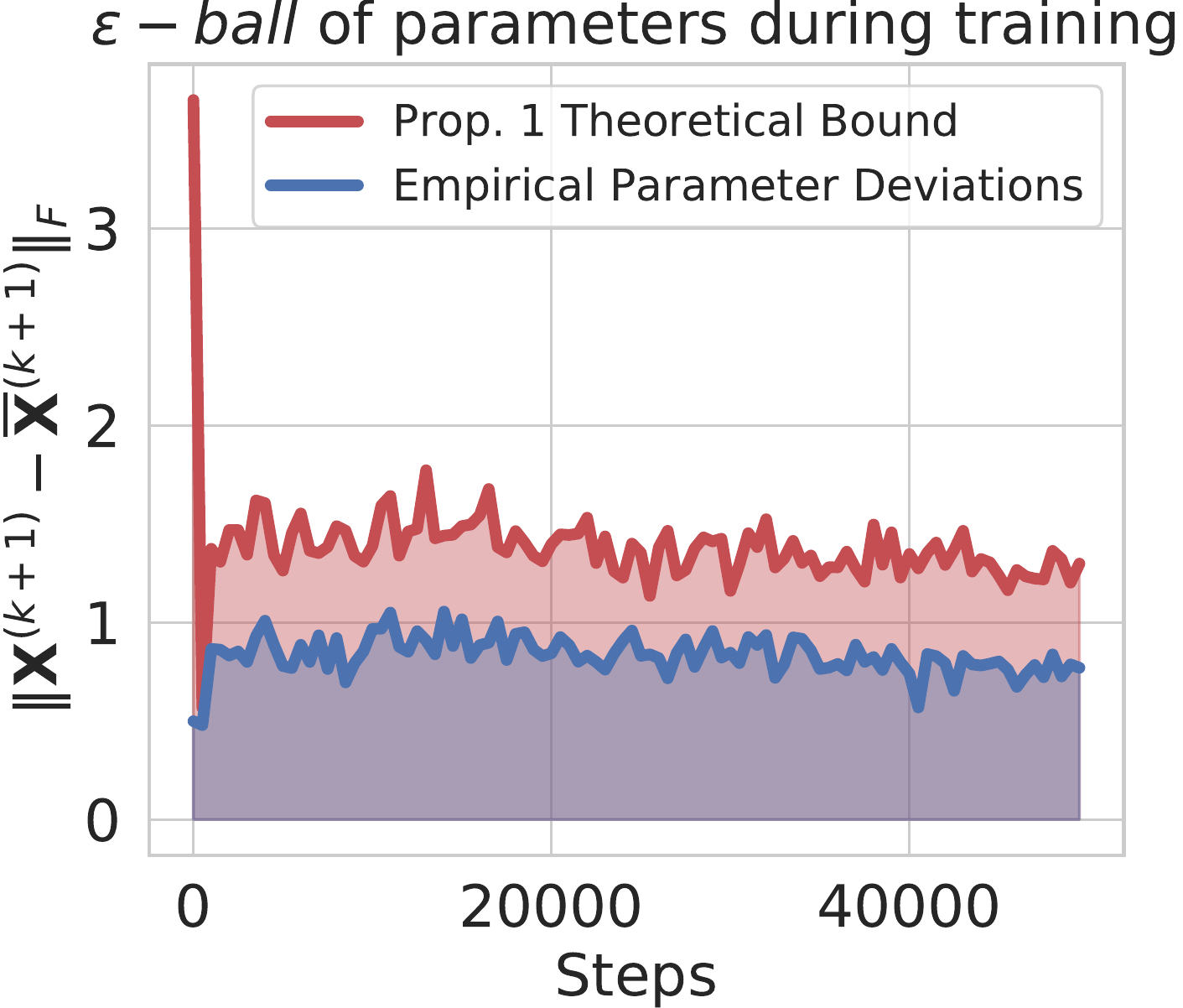}} \quad \quad
    \subfloat[\label{subfig:grad-corr}]{\includegraphics[width=.5\textwidth]{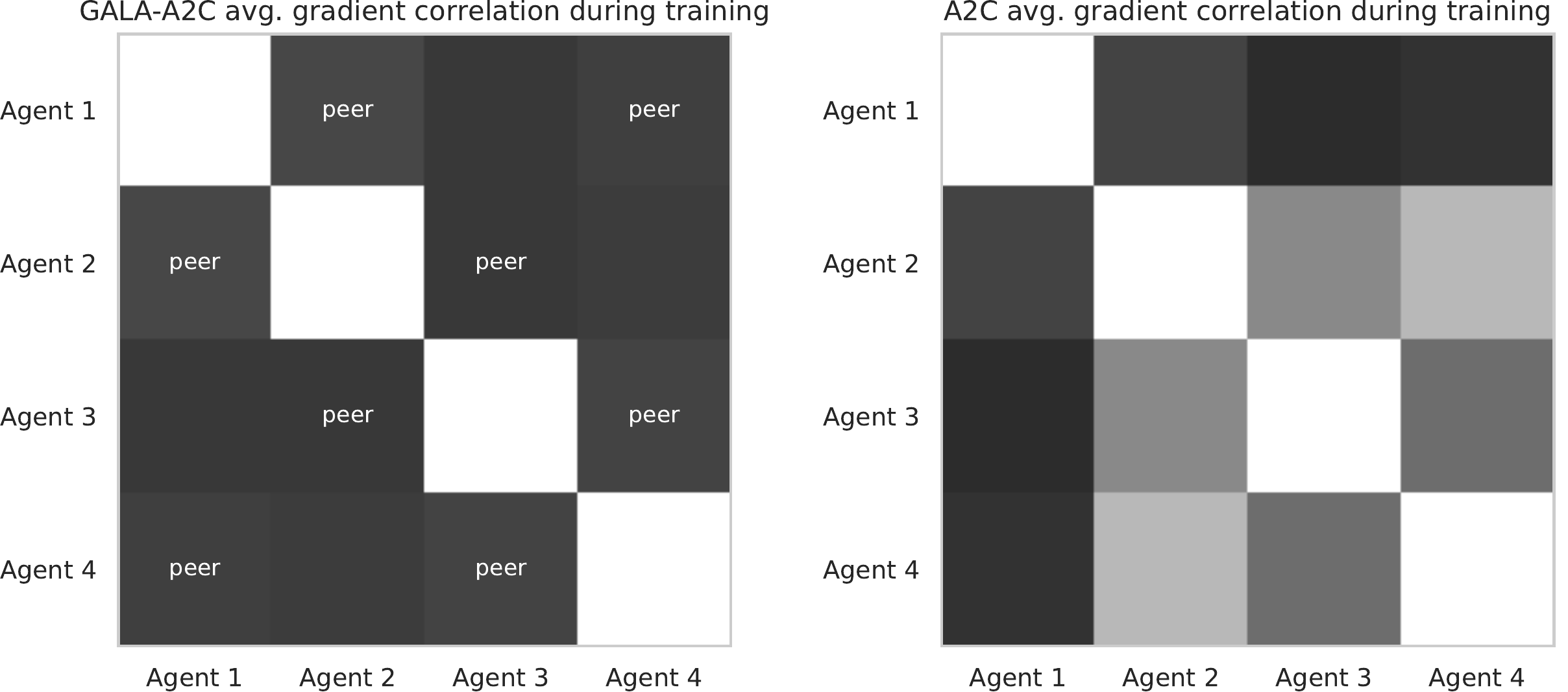}}
    \caption{(a) The radius of the $\epsilon$-ball within which the agents' parameters reside during training.
    The theoretical upper bound in Proposition~\ref{prop:no-ass} is explicitly calculated and compared to the true empirical quantity. The bound in Proposition~\ref{prop:no-ass} is remarkably tight. (b) Average correlation between agents' gradients during training (darker colors depict low correlation and lighter colors depict higher correlations).
    Neighbours in the \GALA-\AtwC topology are annotated with the label ``peer.'' The \GALA-\AtwC heatmap is generally much darker than the \AtwC heatmap, indicating that \GALA-\AtwC agents produce more diverse gradients with significantly less correlation.
    }
    \label{fig:gossip-effects}
\end{figure}
\begin{figure}[t]
    \centering
    \subfloat[\label{subfig:hw-scatter}]{\includegraphics[width=.5\textwidth]{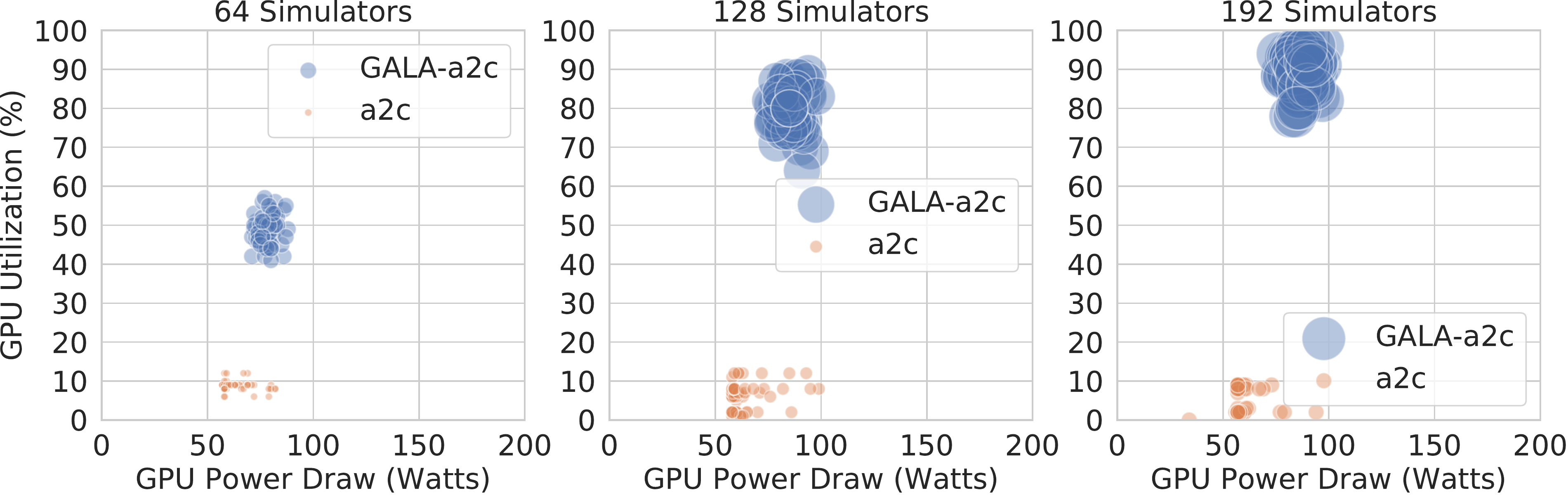}}
    \hfill 
    \subfloat[\label{subfig:hw-bar}]{\includegraphics[width=.5\textwidth]{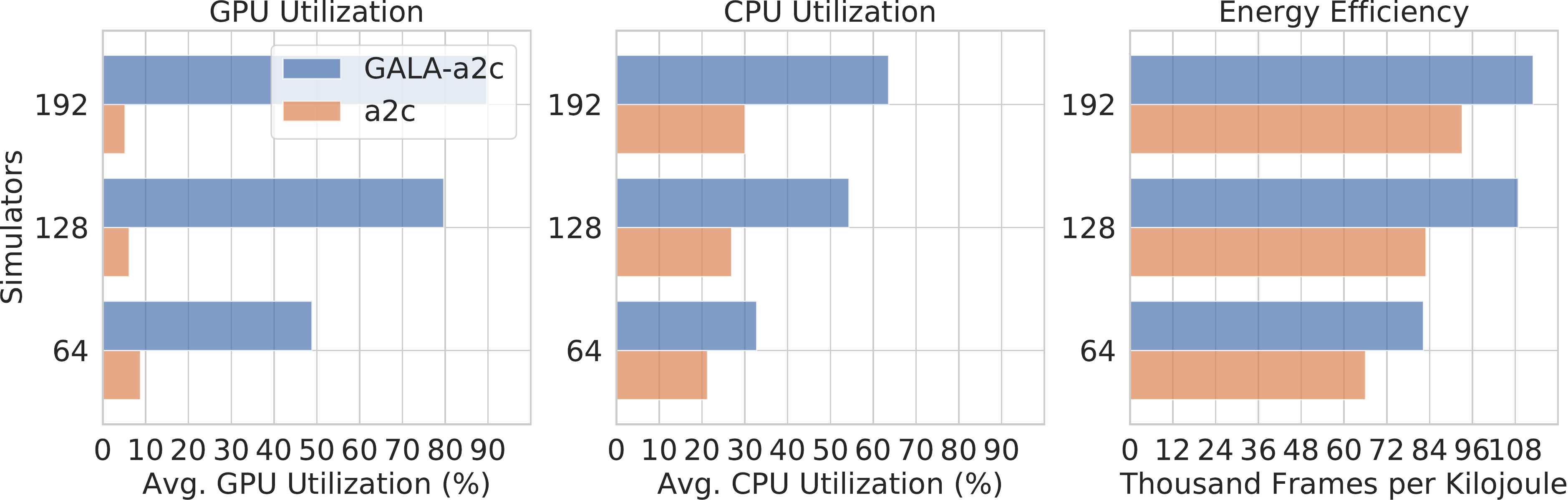}}
    \caption{
    Comparing \GALA-\AtwC hardware utilization to that of \AtwC when using one NVIDIA V100 GPU and 48 Intel CPUs.
    (a) Samples of instantaneous GPU utilization and power draw plotted against each other. Bubble sizes indicate frame-rates obtained by the corresponding algorithms; larger bubbles depict higher frame-rates.
    \GALA-\AtwC achieves higher hardware utilization than \AtwC at comparable power draws.
    This translates to much higher frame-rates and increased energy efficiency.
    (b) Hardware utilization/energy efficiency vs.~number of simulators.
    \GALA-\AtwC benefits from increased parallelism and achieves a 10-fold improvement in GPU utilization over  \AtwC.
    }
    \label{fig:hardware-util}
\end{figure}

Next, we measure the diversity in the agents' gradients.
We hypothesize that the $\epsilon$-diversity in the policies predicted by Proposition~\ref{prop:no-ass}, and empirically observed in Figure~\ref{subfig:epsilon-ball}, may lead to less correlation in the agents' gradients.
The categorical heatmap in Figure~\ref{subfig:grad-corr} shows the pair-wise cosine-similarity between agents' gradients throughout training, computed after every $500$ local environment steps, and averaged over the first 10M training steps.
Dark colors depict low correlations and light colors depict high correlations.
We observe that \GALA-\AtwC agents exhibited less gradient correlations than \AtwC agents.
Interestingly, we also observe that \GALA-\AtwC agents' gradients are more correlated with those of peers that they explicitly communicate with (graph neighbours), and less correlated with those of agents that they do not explicitly communicate with.

\paragraph{Computational performance:}
\label{sec:computational_perf}

Figure~\ref{fig:hardware-util} showcases the hardware utilization and energy efficiency of \GALA-\AtwC compared to \AtwC as we increase the number of simulators.
Specifically, Figure~\ref{subfig:hw-scatter} shows that \GALA-\AtwC achieves significantly higher hardware utilization than vanilla \AtwC at comparable power draws.
This translates to much higher frame-rates and increased energy efficiency.
Figure~\ref{subfig:hw-bar} shows that \GALA-\AtwC is also better able to leverage increased parallelism and achieves a 10-fold improvement in GPU utilization over vanilla \AtwC.
Once again, the improved hardware utilization and frame-rates translate to increased energy efficiency.
In particular, \GALA-\AtwC steps through roughly 20 thousand more frames per Kilojoule than vanilla \AtwC.
Figures~\ref{fig:best_64sim_energy} and~\ref{fig:64sim_energy} compare game scores as a function of energy utilization in Kilojoules.
\GALA-\AtwC is distinctly more energy efficient than the other methods, achieving higher game scores with less energy utilization.


\section{Conclusion}
\label{sec:discussion}

We propose Gossip-based Actor-Learner Architectures (\GALA) for accelerating Deep Reinforcement Learning by leveraging parallel actor-learners that exchange information through asynchronous gossip.
We prove that the \GALA agents' policies are guaranteed to remain within an $\epsilon$-ball during training, and verify this empirically as well.
We evaluated our approach on six Atari games, and find that \GALA-\AtwC improves the computational efficiency of \AtwC, while also providing extra stability and robustness by decorrelating gradients.
\GALA-\AtwC also achieves significantly higher hardware utilization than vanilla \AtwC at comparable power draws, and is competitive with state-of-the-art methods like \AthC and \IMPALA.

\subsubsection*{Acknowledgments}

We would like to thank the authors of TorchBeast for providing their pytorch implementation of \IMPALA.

\bibliography{./NeuRIPS2019/gala.bib}

\newpage
\appendix
\appendixpage
\addappheadtotoc

\section{Reproducibility Checklist}
We follow the reproducibility checklist~\cite{repro_checklist}:
\label{app:repro}
For all algorithms presented, check if you include:
\begin{itemize}
    \item \textbf{A clear description of the algorithm.} See Algorithm~\ref{alg:gala}
    \item \textbf{An analysis of the complexity (time, space, sample size) of the algorithm.} See Figures~\ref{fig:64sim_all} and  ~\ref{fig:hardware-util} for an analysis of sample efficiency, wall-clock time, and energy efficiency. 
    \item \textbf{A link to a downloadable source code, including all dependencies.} See the attached zip file.
\end{itemize}
For any theoretical claim, check if you include:
\begin{itemize}
    \item \textbf{A statement of the result.} See Propositions~\ref{prop:no-ass} and \ref{prop:ass} in the main text. 
    \item \textbf{A clear explanation of any assumptions.} See Appendix \ref{apdx:proofs} for full details.
    \item \textbf{A complete proof of the claim.} See Appendix \ref{apdx:proofs} for full details.
\end{itemize}
For all figures and tables that present empirical results, check if you include:
\begin{itemize}
    \item \textbf{A complete description of the data collection process, including sample size.} We used the Arcade Learning Environment \citep{machado17arcade}, specifically we used the gym package - see: \href{https://github.com/openai/gym}{github.com/openai/gym}
    \item \textbf{A link to downloadable version of the dataset or simulation environment.} See: \href{https://github.com/openai/gym}{github.com/openai/gym}
    \item \textbf{An explanation of how samples were allocated for training / validation / testing.} We didn't use explicit training / validation/ testing splits - but ran each algorithm with 10 different random seeds. 
    \item \textbf{An explanation of any data that were excluded.} We only used $6$ atari games due to time constraints - the same $6$ games that were used in \cite{stooke2018accelerated}.
    \item \textbf{The range of hyper-parameters considered, method to select the best hyper-parameter configuration, and specification of all hyper-parameters used to generate results.} We used standard hyper-parameters from \cite{baselines, stooke2018accelerated, espeholt2018impala}. 
    \item \textbf{The exact number of evaluation runs.} We used 10 seeds for the Atari experiments.
    \item \textbf{A description of how experiments were run.} See Appendix~\ref{app:implementation} for full details.
    \item \textbf{A clear definition of the specific measure or statistics used to report results.} $95\%$ confidence intervals are used in all plots / tables unless otherwise stated.
    \item \textbf{Clearly defined error bars.} $95\%$ confidence intervals are used in all plots / tables unless otherwise stated.
    \item \textbf{A description of results with central tendency (e.g. mean) and variation (e.g. stddev).} $95\%$ confidence intervals are used in all plots / tables unless otherwise stated.
    \item \textbf{A description of the computing infrastructure used.} See Appendix~\ref{app:implementation} for full details.
\end{itemize}

\section{Proofs}
\label{apdx:proofs}

\paragraph{Setting and Notation}
Before presenting the theoretical guarantees, we define some notation. Suppose we have $n$ learners (\eg, actor-critic agents) configured in a peer-to-peer communication topology represented by a directed and potentially time-varying graph (the non-zero entries in the mixing matrix $\bm{P}^{(k)}$ define the communication topology at each iteration $k$).

Learners constitute vertices in the graph, denoted by $v_i$ for all $i \in [n]$ , and edges constitute directed communication links.
Let $\Nout_i$ denote agent $v_i$'s \emph{out-peers}, the set of agents that $v_i$ can send messages to, and let $\Nin_i$ denote agent $v_i$'s \emph{in-peers}, the set of agents that can send messages to $v_i$.
If the graph is time-varying, these sets are annotated with time indices.
Let $\param_i \in \R^{d}$ denote the agent $v_i$'s complete set of trainable parameters, and let the training function $\Tfunc_i : \R^d \mapsto \R^d$ define agent $v_i$'s training dynamics (\ie, agent $v_i$ optimizes its parameters by iteratively computing $\param_i \gets \Tfunc_i(\param_i)$).

For each agent $v_i$ we define send- and receive-buffers, $\Bbuff_i$ and $\Rbuff_i$ respectively, which are used by the underlying communication system (standard in the gossip literature~\citep{tsitsiklis1986distributed}).
When an agent wishes to broadcast a message to its out-peers, it simply copies the message into its broadcast buffer.
Similarly, when agent receives a message, it is automatically copied into the receive buffer.
For convenience, we assume that each learner $v_i$ can hold at most one message from in-peer $v_j$ in its receive buffer, $\Rbuff_i$ at any time $k$; \ie, a newly received message from agent $v_j$ overwrites the older one in the receive buffer.

Let $k \in \N$ denote the global iteration counter. That is, $k$ increments whenever any agent (or subset of agents) completes one loop in Algorithm~\ref{alg:gala}.
Consequently, at each global iteration $k$, there is a set of agents $\mathcal{I}$ that are activated, and within this set there is a (possibly non-empty) subset of agents $\mathcal{C} \subseteq \mathcal{I}$ that gossip in the same iteration.
If a message from agent $v_j$ is received by agent $v_i$ at time $k$, let $\tau_{j,i}^{(k)}$ denote the time at which this message was sent.
Let $\tau \geq \tau_{j,i}^{(k)}$ for all $i,j \in [n]$ and $k > 0$ denote an upper bound on the message delays.
For analysis purposes, messages are sent with an effective delay such that they arrive right when the agent is ready to process the messages. That is, a message that is sent by agent $v_j$ at iteration $k^\prime$ and processed by agent $v_i$ at iteration $k$, where $k \geq k^\prime$, is treated as having
experienced a delay $\tau_{j,i}^{(k)} = k - k^\prime$, even if the message actually arrives before iteration $k$ and waits in the receive-buffer.

Let $\alpha g_i^{(k)} \defeq \Tfunc_i(\param_i^{(k)}) - \param_i^{(k)}$ denote agent $v_i$'s local computation update at iteration $k$ after scaling by some reference learning rate $\alpha > 0$, and define $g_i^{(k)} \defeq 0$ if agent $v_i$ is not active at iteration $k$.
Algorithm~\ref{alg:gala} can thus be written as follows.
If agent $v_i$ does not gossip at iteration $k$, then its update is simply
\begin{equation}
    \param_i^{(k+1)} = \param_i^{(k)} + \alpha g_i^{(k)}.
\end{equation}
If agent $v_i$ does gossip at iteration $k$, then its update is
\begin{equation}
\label{eq:gala-itr}
    \param_i^{(k + 1)} = \frac{1}{1 + |\Nin_i|} \left(\param_i^{(k)} + \sum_{j \in \Nin_i} \param_j^{\tau_{j,i}^{(k)}} + \alpha g_i^{(k)} \right),
\end{equation}
where $\param_j^{\tau_{j,i}^{(k)}}$ is the parameter value of the agent $v_j$, at the time where the message was sent, \ie, $\tau_{j,i}^{(k)}$.

We can analyze Algorithm~\ref{alg:gala} in matrix form by stacking all $n$ agents' parameters, $\param_i^{(k)} \in \R^d$, into a matrix $\bm{X}^{(k)}$, and equivalently stacking all of the update vectors, $g_i^{(k)} \in \R^d$, into a matrix $\bm{G}^{(k)}$.
In order to represent the state of messages that are in transit (sent but not yet received), for analysis purposes, we augment the graph topology with virtual nodes using a standard graph augmentation~\citep{hadjicostis2013average} (we add $\tau$ virtual nodes for each non-virtual agent, where each virtual node stores a learner's parameters at a specific point within the last $\tau$ iterations).
Let $\naug \defeq n(\tau + 1)$ denote the cardinality of the augmented graph's vertex set.
Equation~\eqref{eq:gala-itr} can be re-written as
\begin{equation}
\label{eq:gala-matrix-itr}
    \bm{X}^{(k + 1)} = \bm{\tilde{P}}^{(k)}\left( \bm{X}^{(k)} + \alpha \bm{G}^{(k)} \right),
\end{equation}
where $\bm{X}^{(k)}, \bm{G}^{(k)} \in \R^{\naug \times d}$, and the mixing matrix $\bm{\tilde{P}}^{(k)} \in \R^{\naug \times \naug}$ corresponding to the augmented graph is row-stochastic for all iterations $k$, \ie, all entries are non-negative, and all rows sum to $1$ .
Mapping~\eqref{eq:gala-itr} to~\eqref{eq:gala-matrix-itr} may not be obvious, but is quite standard in the recent literature.
We refer the interested reader to~\cite{assran2018asynchronous, hadjicostis2013average}.

\begin{proof}[Proof of Proposition~\ref{prop:no-ass}]
The proof is very similar to the proofs in~\cite{assran2018stochastic} and~\cite{assran2018asynchronous}, and makes use of the graph augmentations in~\cite{hadjicostis2013average}, the lower dimensional stochastic matrix dynamics in~\cite{blondel2005convergence}, and the ergodic matrix results in~\cite{wolfowitz1963products}.
Since the matrices $\bm{\tilde{P}}^{(k)}$ are row-stochastic, their largest singular value is $1$, which corresponds to singular vectors in $\sspan{1_{\naug}}$.
Let the matrix $\bm{Q} \in \R^{(\naug - 1) \times \naug}$ define an orthogonal projection onto the space orthogonal to $\sspan{1_{\naug}}$.
Associated to each matrix $\bm{\tilde{P}}^{(k)} \in \R^{\naug \times \naug}$ there is a unique matrix $\bm{\tilde{P}}^{\prime (k)} \in \R^{(\naug - 1) \times (\naug - 1)}$ such that $\bm{Q} \bm{\tilde{P}}^{(k)} = \bm{\tilde{P}}^{\prime (k)} \bm{Q}$.
Let $\bm{\tilde{P}}^{\prime}$ denote the collection of matrices $\bm{\tilde{P}}^{\prime (k)}$ for all $k$.
The spectrum of the matrices $\bm{\tilde{P}}^{\prime (k)}$ is the spectrum of $\bm{\tilde{P}}^{(k)}$ after removing one multiplicity of the singular value $1$.
From~\eqref{eq:gala-matrix-itr}, we have
\begin{align}
\label{eq:prf-prop-1}
\begin{split}
    \bm{Q} \bm{X}^{(k + 1)} =& \bm{Q} \bm{\tilde{P}}^{(k)} \cdots \bm{\tilde{P}}^{(1)} \bm{\tilde{P}}^{(0)} \bm{X}^{(0)} + \alpha \sum^{k}_{s=0} \bm{Q} \bm{\tilde{P}}^{(k)} \cdots \bm{\tilde{P}}^{(s + 1)} \bm{\tilde{P}}^{(s)} \bm{G}^{(s)} \\
    =& \bm{\tilde{P}}^{\prime (k)} \cdots \bm{\tilde{P}}^{\prime (1)} \bm{\tilde{P}}^{\prime (0)} \bm{Q} \bm{X}^{(0)} + \alpha \sum^{k}_{s=0} \bm{\tilde{P}}^{\prime (k)} \cdots \bm{\tilde{P}}^{\prime (s + 1)} \bm{\tilde{P}}^{\prime (s)} \bm{Q} \bm{G}^{(s)}.
\end{split}
\end{align}
Note that $\bm{Q} (\bm{X}^{(k + 1)} - \overline{\bm{X}}^{(k+1)}) = 0$ and $(\bm{X}^{(k + 1)} - \overline{\bm{X}}^{(k+1)})^T 1_{\naug} = 0$.
Thus
\begin{align*}
    \norm{ \bm{X}^{(k + 1)} - \overline{\bm{X}}^{(k+1)} } =& \norm{\bm{Q}(\bm{X}^{(k + 1)} - \overline{\bm{X}}^{(k+1)})} \\
    \leq& \norm{\bm{\tilde{P}}^{\prime (k)} \cdots \bm{\tilde{P}}^{\prime (s + 1)} \bm{\tilde{P}}^{\prime (0)} \bm{Q} \bm{X}^{(0)}} + \alpha \sum^{k}_{s=0} \norm{ \bm{\tilde{P}}^{\prime (k)} \cdots \bm{\tilde{P}}^{\prime (s + 1)} \bm{\tilde{P}}^{\prime (s)} \bm{Q} \bm{G}^{(s)} },
\end{align*}
where we have implicitly also made use of~\eqref{eq:prf-prop-1}.

Defining $\beta \defeq \sup_{s=0,1,\ldots, k} \sigma_2( \bm{\tilde{P}}^{\prime (k)} \cdots \bm{\tilde{P}}^{\prime (s + 1)} \bm{\tilde{P}}^{\prime (s)})$, it follows that
\begin{align}
\label{eq:prf-prop-2}
    \norm{ \bm{X}^{(k + 1)} - \overline{\bm{X}}^{(k+1)} } \leq& \beta^{k + 1} \norm{\bm{Q} \bm{X}^{(0)}} + \alpha \sum^{k}_{s=0} \beta^{k + 1 - s} \norm{\bm{Q} \bm{G}^{(s)} }.
\end{align}
Assuming all learners are initialized with the same parameters, the first exponentially decay term on the right hand side of~\eqref{eq:prf-prop-2} vanishes and we have
\begin{align*}
    \norm{ \bm{X}^{k + 1} - \overline{\bm{X}}^{(k+1)}} \leq \alpha \sum^{k}_{s=0} \beta^{k + 1 - s} \norm{\bm{G}^{(s)}}.
\end{align*}
\end{proof}

\begin{proof}[Proof of Proposition~\ref{prop:ass}]
The proof extends readily from Proposition~\ref{prop:no-ass}.
Given the assumptions on the graph sequence, the product of the matrices $\bm{\tilde{P}}^{(k)} \cdots \bm{\tilde{P}}^{(s + 1)} \bm{\tilde{P}}^{(s)}$ is ergodic for any $k - s \geq \tau + B$ (cf.~\citep{assran2018masters-asynchronous}).
Letting
$\beta \defeq \sup_{s=0,1,\ldots, k} \sigma_2( \bm{\tilde{P}}^{\prime (k)} \cdots \bm{\tilde{P}}^{\prime (s)})$, it follows from~\cite{wolfowitz1963products} and~\cite{blondel2005convergence} that $\beta < 1$.
\end{proof}

\section{Implementation Details}
\label{app:implementation}

\subsection{\AthC Implementation Comparison}
\label{app:a3c}
\begin{figure}[t]
    \centering
    \subfloat[Sample complexity: Average across 10 runs.]{\includegraphics[width=1.\textwidth]{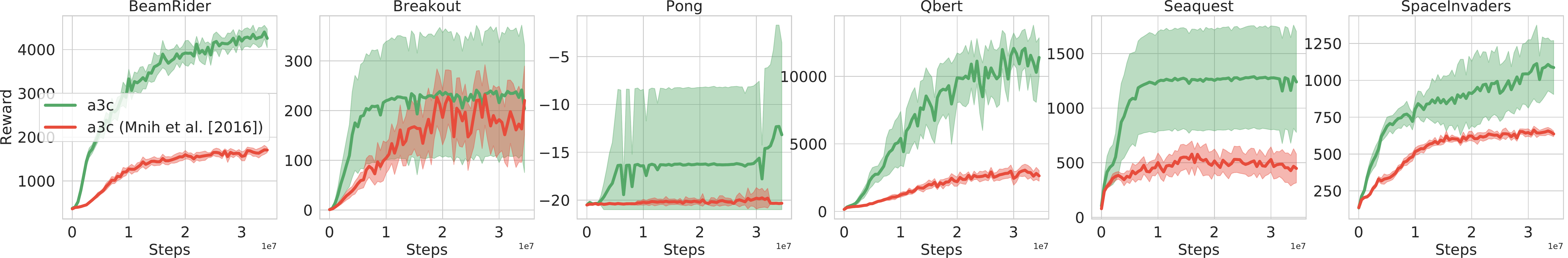}} \\
    \subfloat[Computational complexity: Average across 10 runs.]{\includegraphics[width=1.\textwidth]{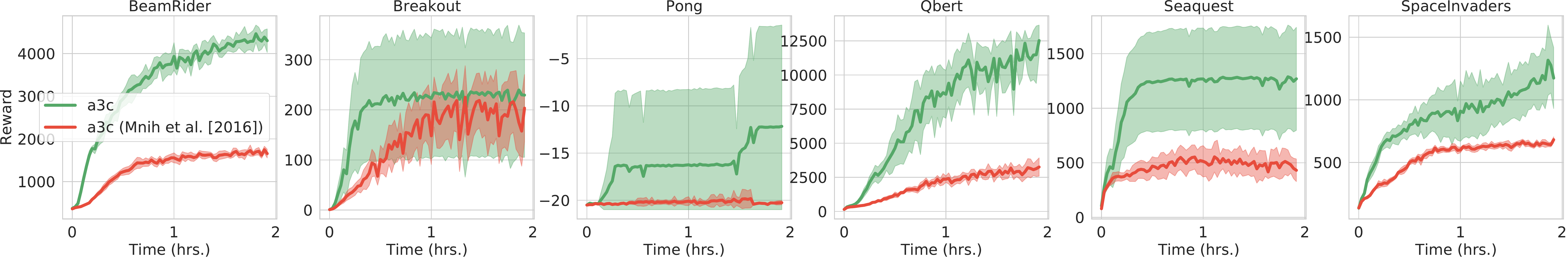}}
    \caption{Comparing large-batch \AthC~\citep{stooke2018accelerated} to original \AthC~\citep{mnih2016asynchronous}.
    Running \AthC with larger-batches provides more stable and sample efficient convergence (top row), while also maintaining computational efficiency by leveraging GPU acceleration (bottom row).}
    \label{fig:a3c_variants}
\end{figure}
While \AthC was originally proposed with CPU-based agents with 1-simulator per agent, \citet{stooke2018accelerated} propose a variant in which each agent manages 16-simulators and performs batched inference on a GPU.
Figure~\ref{fig:a3c_variants} compares 64-simulator learning curves using \AthC as originally proposed in~\citet{mnih2016asynchronous} to the large-batch variant in~\citet{stooke2018accelerated}.
The large-batch variant appears to be more robust and computationally efficient, therefore we use this GPU-based version of \AthC in our main experiments to provide a more competitive baseline.

\subsection{Experimental Setup}
All experiments use the network suggested by~\citet{baselines}.
Specifically, the network contains $3$ convolutional layers and one hidden layer, followed by a linear output layer for the policy/linear output layer for the critic.
The hyper-parameters for \AtwC, \AthC and \GALA-\AtwC are summarized in Table~\ref{tab:hyperparameters}. \IMPALA hyperparameters are the same as reported in~\cite{espeholt2018impala} (cf.~table $G.1$ in their appendix).
\begin{table}[!htbp]
    \centering
    \begin{tabular}{c | c}
    \hline
         \textbf{Hyper-parameter} & \textbf{Value} \\
         \hline \\
         Image Width & $84$ \\
         Image Height & $84$ \\
         Grayscaling & Yes \\
         Action Repetitions & $4$ \\
         Max-pool over last $k$ action repeat frames & $2$ \\
         Frame Stacking & $4$ \\ 
         End of episode when life lost & Yes \\
         Reward Clipping & $[-1, 1]$ \\
         RMSProp momentum & $0.0$ \\ 
         RMSProp $\epsilon$ & $0.01$ \\
         Clip global gradient norm & $0.5$ \\
         Number of simulators per learner($B$) & $16$\\
         Base Learning Rate($\alpha$) & $7 \times 10^{-4}$ \\
         Learning Rate Scaling & $\sqrt{\text{number of learners}}$  (not used in \AthC)\\
         VF coeff & $0.5$\\
         Horizon(N) & $5$\\
         Entropy Coeff. & $0.01$\\
         Discount ($\gamma$)& .99\\
         Max number of no-ops at the start of the episode & $30$ \\
    \hline \\
    \end{tabular}
    \caption{Hyperparameters for both \AtwC, \GALA-\AtwC, and \AthC}
    \label{tab:hyperparameters}
\end{table}

In all \GALA experiments we used $16$ environments per learner, \eg, in the $64$ simulator experiments in Section~\ref{sec:computational_perf} we use $4$ learners. \GALA agents communicate using a 1-peer ring network . Figure~\ref{fig:graph_topology} shows an example of such a ring network. The non-zero weight $p_{i,j}$ of the mixing matrix $\bm{P}$ corresponding to the 1-peer ring are set to $\frac{1}{1 + |\Nin_i|}$, which is equal to $1/2$ as $|\Nin_i| = 1$ for all $i$ in the 1-peer ring graph.

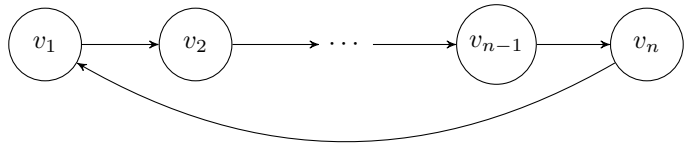
\begin{figure}[h]

\begin{center}
\begin{tikzpicture}[->,>=stealth',auto,node distance=2cm,
  main node/.style={circle,draw,minimum size=2.75em}]

  \node[main node] (1) {$v_1$};
  \node[main node] (2) [right of=1] {$v_2$};
  \node (3) [right of=2] {$\cdots$};
  \node[main node] (4) [right of=3] {$v_{n-1}$};
  \node[main node] (5) [right of=4] {$v_{n}$};

  \path[]
    (1) edge node [right] {} (2)
    (2) edge node [right] {} (3)
    (3) edge node [right] {} (4)
    (4) edge node [right] {} (5)
    (5) edge[bend left] node [left] {} (1);
\end{tikzpicture}
	\caption{Example of an n-agents/1-peer ring communication topology used in our experiment }
	\label{fig:graph_topology}
\end{center}
\end{figure}


\end{document}